\documentclass[11pt]{article}

\usepackage{fullpage}
\usepackage{natbib}
\renewcommand\cite{\citep}

\usepackage[utf8]{inputenc} 
\usepackage{hyperref}       
\usepackage{url}            
\usepackage{booktabs}       
\usepackage{amsfonts}       
\usepackage{nicefrac}       
\usepackage{amsmath, amsthm}
\usepackage{algorithm, algorithmicx, algpseudocode}
\usepackage{subfigure}
\usepackage{amssymb}

\usepackage{color}
\usepackage[dvipsnames]{xcolor}

\newcommand{\cM}{\mathcal{M}}
\newcommand{\cS}{\mathcal{S}}
\newcommand{\cA}{\mathcal{A}}
\newcommand{\cF}{\mathcal{F}}
\newcommand{\ts}{\widetilde{s}}
\newcommand{\RR}{\mathbb{R}}
\newcommand{\PP}{\mathbb{P}}
\newcommand{\EE}{\mathbb{E}}

\newcommand{\cH}{\mathcal{H}}
\newcommand{\cB}{\mathcal{B}}

\newcommand{\wt}{\widetilde}
\newcommand{\wb}{\overline}

\newcommand{\cK}{\mathcal{K}}
\newcommand{\var}{\mathrm{Var}}
\newcommand{\tr}{\mathrm{tr}}
\newcommand{\poly}{\mathrm{poly}}
\newcommand{\tw}{\widetilde{w}}
\newcommand{\reg}{\mathrm{Regret}}

\def\cX{\mathcal{X}}

\ifx\theorem\undefined
\newtheorem{theorem}{Theorem}
\fi
\ifx\example\undefined
\newtheorem{example}{Example}
\fi
\ifx\property\undefined
\newtheorem{property}{Property}[theorem]
\fi
\ifx\lemma\undefined
\newtheorem{lemma}[theorem]{Lemma}
\fi
\ifx\proposition\undefined
\newtheorem{proposition}{Proposition}
\fi
\ifx\remark\undefined
\newtheorem{remark}{Remark}
\fi
\ifx\corollary\undefined
\newtheorem{corollary}[theorem]{Corollary}
\fi
\ifx\definition\undefined
\newtheorem{definition}{Definition}
\fi
\ifx\conjecture\undefined
\newtheorem{conjecture}[theorem]{Conjecture}
\fi
\ifx\fact\undefined
\newtheorem{fact}[theorem]{Fact}
\fi
\ifx\claim\undefined
\newtheorem{claim}[theorem]{Claim}
\fi
\ifx\assump\undefined
\newtheorem{assumption}{Assumption}
\fi
\ifx\cond\undefined
\newtheorem{cond}[theorem]{Condition}
\fi
\ifx\problem\undefined
\newtheorem{problem}{Problem}
\fi

\newtheorem{innercustomassumption}{Assumption}
\newenvironment{customassumption}[1]
{\renewcommand\theinnercustomassumption{#1}\innercustomassumption}
{\endinnercustomassumption}

\title{Reinforcement Learning in Feature Space: Matrix Bandit, Kernels, and Regret Bound}

\author{%
	Lin F. Yang\\
	Princeton University\\
	\texttt{lin.yang@princeton.edu}\\
	\and
	Mengdi Wang\\
	Princeton University\\
	\texttt{mengdiw@princeton.edu}
}

\begin{document}
	\maketitle
	
\begin{abstract}
Exploration in reinforcement learning (RL) suffers from the curse of dimensionality when the state-action space is large. 
A common practice is to parameterize the high-dimensional value and policy functions using given features.
However existing methods either have no theoretical guarantee or suffer a regret that is exponential in the planning horizon $H$.
In this paper, 
we propose an online RL algorithm, namely the MatrixRL, that leverages ideas from linear bandit  to learn a low-dimensional representation of the probability  transition model while carefully balancing the exploitation-exploration tradeoff.
We show that MatrixRL achieves a regret bound ${O}\big(H^2d\log T\sqrt{T}\big)$ where $d$ is the number of features. 
MatrixRL has an equivalent kernelized version, which is able to work with an arbitrary kernel Hilbert space without using explicit features.
In this case, the kernelized MatrixRL satisfies a regret bound ${O}\big(H^2\wt{d}\log T\sqrt{T}\big)$, where $\wt{d}$ is the effective dimension of the kernel space.
To our best knowledge,  for RL using features or kernels, our results are the first regret bounds that are near-optimal in time $T$ and dimension $d$  (or $\wt{d}$) and polynomial in the planning horizon $H$.
\end{abstract}

\section{Introduction}
Reinforcement learning (RL) is about learning to make sequential decisions in an unknown environment through trial and error. It finds wide applications in robotics \cite{kober2013reinforcement}, autonomous driving \cite{shalev2016safe}, game AI \cite{silver2017mastering} and beyond. 
We consider a basic RL model - the Markov decision process (MDP). In the MDP, an agent at a state $s\in\cS$ is able to play an action $a\in \cA$, where $\cS$ and $\cA$ are the state and action spaces. Then 
the system transitions to another state $s'\in \cS$ according to an unknown probability $P(s'~|~s,a)$, while returning an immediate reward $r(s,a)\in[0,1]$. 
The goal of the agent is to obtain the maximal possible return after playing for a period of time - even though she has no knowledge about the transition probabilities at the beginning. 

The performance of a learning algorithm is measured by ``regret". Regret is the difference between the cumulative reward obtained using the best possible policy and the cumulative reward obtained by the learning algorithm. 
In the tabular setting where $\cS$ and $\cA$ are finite sets, there exist algorithms that achieve asymptotic regret $ O(\sqrt{|\cS||\cA|T})$ (e.g. \cite{jaksch2010near,osband2016lower, osband2017deep, agrawal2017optimistic, azar2017minimax, dann2018policy,jin2018q}), where $T$ is the number of time steps. 
However, the aforementioned regret bound depends polynomially on ${|\cS|}$ and $|\cA|$, sizes of the state and action space, which can be very large or even infinite.
For instance, the game of Go has $3^{361}$ unique states, and a robotic arm has infinitely many continuous-valued states.
In the most general sense,  the regret $O(\sqrt{|\cS||\cA|T})$ is nonimprovable in the worst case \cite{jaksch2010near}. 
This issue is more generally known as the ``curse of dimensionality'' of control and dynamic programming \cite{bellman1966dynamic}.

To tackle the dimensionality, a common practice is to use features to parameterize high-dimensional value and policy functions in compact presentations, with the hope that the features can capture leading structures of the MDP. 
In fact, there are phenomenal empirical successes of reinforcement learning using explicit features and/or neural networks as implicit features (see e.g., \cite{mnih2015human}).
However, there is a lack of theoretical understanding about using features for exploration in RL and its learning complexity.
In this paper, we are interested in the following theoretical question:
\[
\text{\it {How to use features for provably efficient exploration in reinforcement learning?}}
\]
Furthermore, we consider online RL in a reproducing kernel space. Kernel methods are well known to be powerful to capture nonlinearity and high dimensionality in many machine learning tasks \cite{shawe2004kernel}. We are interested in using \emph{kernel methods} to capture nonlinearity in the state-transition dynamics of MDP.  A kernel space may consist of infinitely many implicit feature functions. We study the following questions:
How to use kernels in online reinforcement learning? Can one achieve low regret even though the kernel space is infinite-dimensional?
The goal of this paper is to answer the aforementioned questions affirmatively.  In particular, we would like to design algorithms that take advantages of given features and kernels to achieve efficient exploration. 



\subsection{Our Approach and Main Results}

Consider episodic reinforcement learning in finite-horizon MDP. The agent learns through episodes, and each episode consists of $H$ time steps. 
Here $H$ is also called the \emph{planning horizon}. 
Let $\phi(\cdot)\RR^d, \psi(\cdot)\in \RR^{d'}$ be given feature functions.
We focus on the case that the probability transition model $P(\cdot~|~\cdot)$ can be fully embedded in the feature space (Assumption 1), i.e., there exists some core matrix $M^*$ such that
 \[
 P(\cdot~|~\cdot) = \phi(\cdot)^\top M^* \psi(\cdot).
 \]
 In the kernel setting, this condition is equivalent to that the transition probability model $P$ belongs to the product space of  the reproducing kernel spaces.
 This condition is essentially equivalent to using the features $\phi$ to represent value functions \cite{parr2008analysis}. When the probability transition model $P$ cannot be fully embedded using $\phi$, then value function approximation using $\phi$ may lead to arbitrarily large Bellman error \cite{yang2019sample}. 
 
We propose an algorithm, which is referred to as MatrixRL, that actively explores the state-action space by estimating the core matrix via ridge regression. The algorithm balances the exploitation-exploration tradeoff by constructing a confidence ball of core matrix for optimistic dynamic programming. It can be thought of as a ``matrix bandit" algorithm which generalizes the idea of linear bandit (e.g. \cite{dani2008stochastic,li2010contextual, Chu2011}).
It is proved to achieve the regret bound either 
\[
\wt{O}(H^2d^{3/2}\sqrt{T}) \quad\text{or}\quad  \wt{O}(H^2d\sqrt{T})\footnote{$\wt{O}(\cdot)$ hides $\poly\log$ factors of the input.},
\] depending on regularity properties of the features. MatrixRL can be implemented efficiently in space $O(d^2)$. Each step can be carried out in closed form.
Next we extend the MatrixRL to work with the kernel spaces with $k_{\phi}(\cdot, \cdot) =\langle \phi(\cdot), \phi(\cdot)\rangle$ and $k_{\psi}(\cdot, \cdot) =\langle \psi(\cdot), \psi(\cdot)\rangle$, and show that it admits a kernelized version. 
The kernelized MatrixRL achieves a regret bound of 
\[
\wt{O}(H^2\wt{d}\sqrt{T})
\]
where $\wt{d}$ is the \emph{effective dimension} of kernel space, even if there may be infinitely many features.
The regret bounds using features or kernels do not depend on sizes of the state and action spaces, making efficient exploration possible in high dimensions.

Note that for linear bandit, the regret lower bound is known to be $\wt{\Omega}(d\sqrt{T})$  \cite{dani2008stochastic}. Since linear bandit is a special case of RL, 
our regret bounds match the lower bound up to polylog factors in $d$ and $T$.  
To our best knowledge, for reinforcement learning using features/kernels, our result gives the first regret bound that is \emph{simultaneously near-optimal in time $T$,  polynomial in the planning horizon $H$, and near-optimal in the feature dimension $d$.}


\subsection{Related Literature}

In the tabular case where there are finitely many states and actions without any structural knowledge, complexity and regret for RL has been extensively studied.
For $H$-horizon episodic RL, efficient methods typically achieve regret that scale asymptotically as $O(\sqrt{HSAT})$ (see for examples \cite{jaksch2010near,osband2016lower, osband2017deep, agrawal2017optimistic,azar2017minimax,dann2018policy,jin2018q}). In particular, \cite{jaksch2010near} provided a regret lower bound  $\Omega(\sqrt{HSAT})$ for $H$-horizon MDP. 
 There is also a line of works studying the sample complexity of obtaining a value or policy that is at most $\epsilon$-suboptimal \cite{kakade2003sample, strehl2006pac,strehl2009reinforcement,szita2010model, lattimore2014near,azar2013minimax, dann2015sample,sidford2018near}. The optimal sample complexity for finding an $\epsilon$-optimal policy is $O\big(|\cS||\cA|(1-\gamma)^{-2}\epsilon^{-2}\big)$ \cite{sidford2018near} for a discounted MDP with discount factor $\gamma$. The optimal lower bound has been proven in \cite{azar2013minimax}.

There is also a line of works on solving MDPs with a function approximation. For instance \cite{baird1995residual,  tsitsiklis1997analysis, parr2008analysis, mnih2013playing, mnih2015human, silver2017mastering, yang2019sample}.
There are also phenomenal empirical successes in deep reinforcement learning as well (e.g., \cite{silver2017mastering}). 
However there are not many works on the regret analysis of RL with function approximators. 
Very recently, \cite{azizzadenesheli2018efficient} studied the regret bound for linear function approximator. 
However their bound has factor that can be exponential in $H$.
\cite{pmlr-v89-chowdhury19a} considers the regret bound for kernelized MDP. However, they need a Gaussian process prior and assumes that the transition is deterministic with some controllable amount of noise -- a very restrictive setting.
Another work \cite{modi2019contextual} also considers the linear setting for RL. However, the regret bound  is linearly depending on the number of states. 
To the best of our knowledge, we are not aware of other works that achieve regret bound for RL with function approximators that is simultaneously near optimal in $T$, polynomial in $H$, and has no dependence with the state-action space size.

Our results are also related to the literature of linear bandits. 
Bandit problems can be viewed as a special case as Markov decision problems.
There is a line of works on linear bandit problems and their regret analysis \cite{dani2008stochastic, rusmevichientong2010linearly, li2010contextual, abbasi2011improved, Chu2011}. For a more detailed survey, please refer to \cite{bubeck2012regret}. 
Part of our results are inspired by the kernelization for the linear bandit problems, e.g. \cite{Valko, chowdhury2017kernelized}, who studied the regret bound when the features of each arm lies in some reproducing kernel Hilbert space.



\section{Problem Formulation}

In a \emph{episodic Markov decision process} (MDP for short),
there is a set of \emph{states} $\cS$ and a set of \emph{actions} $\cA$, which are not necessarily finite. 
At any state $s\in \cS$, an agent is allowed to play an action $a\in \cA$.
She receives an immediate reward $r(s, a)\in [0, 1]$ after playing $a$ at $s$,
 the process will transition to the next state $s'\in \cS$ with probability $P(s'~|~s,a)$, where $P$ is the collection of \emph{transition distributions}.
After $H$ time steps, the system restarts at a prespecified state $s_0$.
The full instance of an MDP can be described by the tuple $M=(\cS,\cA,P,r,s_0, H).$
The agent would like to find a \emph{policy} $\pi:\cS\times [H]\rightarrow \cA$ that maximizes the long-term expected reward starting from every state $s$ and every stage $h\in[H]$, i.e.,
\[
V_h^{\pi}(s) := \EE\bigg[\sum_{h'= h}^{H}r(s^t, \pi_h(s^t))|s^0 = s\bigg].
\]
We call $V^{\pi}:{\cS\times[H]}\rightarrow \RR$ the \emph{value function} of policy $\pi$.
A policy $\pi^*$ is said to be \emph{optimal} if it attains the maximal possible value at every state-stage pair $(s,h)$.
We denote $V^*$ as the optimal value function. We also denote the optimal action-value function (or $Q$-function) as
\[
\forall h\in [H-1]:\quad 
Q_h^*(s,a) = r(s,a) + P(\cdot|s,a)^\top V_{h+1}^*,
\]
and $Q_{H}^*(s,a) = r(s,a)$.

In the online RL setting, the learning algorithm interacts with the environment episodically.
Each episode starts from state $s_0$ takes $H$ steps to finish. 
We let $n$ denote the current number of episodes and denote $t = (n-1)H + h$ the current time step. 
We equalize $t$ and $(n,h)$ and may switch between the two nations.
We use the following definition of regret.
\begin{definition}
	Suppose we run algorithm $\cK$ in the online environment of an MDP $\cM=(\cS,\cA, P,r, s_0, H)$ for $T=NH$ steps. 
	We define the regret for algorithm $\cK$ as
	\begin{align}
	\label{eqn:regret-def}
	\reg_{\cK}(T) = \EE_{\cK}\bigg[\sum_{n=1}^N\bigg( V^*(s_0) - \sum_{h=1}^Hr\Big(s_{t}, a_{t}\Big)\bigg)\bigg],
	\end{align}
	where $\EE_{\cK}$ is taken over the random path of states under the control of algorithm $\cK$.
\end{definition}

Throughout this paper, we focus on RL problems where the probability transition kernel $P(\cdot~|~\cdot)$ can be fully embedded in a given feature space. 

\begin{assumption}[Feature Embedding of Transition Model]
	\label{assump:linear transition core}
	For each $(s,a)\in \cS\times\cA, \ts\in \cS$, feature vectors $\phi(s,a)\in\RR^d, \psi(\ts)\in \RR^{d'}$ are given as a priori. 
	There exists an unknown matrix $M^*\in \RR^{d\times d'}$ such that 
	\[
	P(\ts~|~s,a)= \phi(s,a)^\top M^*\psi(\ts).
	\]
	Here, we call the matrix $M^*$ as a \emph{transition core}.
\end{assumption}

Note that when $\phi,\psi$ are features associated with two reproducing kernel spaces $\cH_1$ and $\cH_2$, this assumption requires that $P$ belong to their product kernel space $\cH_1\times \cH_2$. 

For simplicity of representation, we assume throughout that the reward function $r$ is known. This is in fact without loss of generality because learning about the environment $P$ is much harder than learning about $r$. In the case if $r$ is unknown but satisfies $\EE[r(s,a)~|~s,a] = \phi(s,a)^\top \theta^*$ for some unknown $\theta^*\in \RR^d$, we can extend our algorithm by adding a step of optimistic reward estimation like in LinUCB \cite{dani2008stochastic, Chu2011}. This would generate an extra $\wt{O}(d\sqrt{T})$ regret, which is a low order term compared to our current regret bounds.

\section{RL Exploration in Feature Space}
\label{sec:rl-feature}

In this section, we 
study the near optimal way to balance exploration and exploitation in RL using a given set of features. We aim to develop an online RL algorithm with regret that depends only on the feature size $d$ but not on the size of the state-action space. Our algorithm is inspired by the LinUCB algorithm \cite{Chu2011} and its variants \cite{dani2008stochastic} and can be viewed as a ``matrix bandit" method.



\subsection{The MatrixRL Algorithm}
The high level idea of the algorithm is to approximate the unknown transition core $M^*$ using data that has been collected so far.
Suppose at the time step $t=(n,h)$ (i.e. episode $n\le N$ and stage $h\le H$), we obtain the following state-action-state transition triplet:
$
(s_{t}, a_{t}, \ts_{t}),
$
where $\ts_t := s_{t+1}$.
For simplicity, we denote the associated features by
\[
\phi_{t}:= \phi(s_{t}, a_{t}) \in\mathbb{R}^{d}
\quad\text{and}\quad
\psi_{t}:= \psi(\ts_{t})  \in\mathbb{R}^{d'}.
\]

\paragraph{Estimating the core matrix.}
Let $K_{\psi} := \sum_{\ts\in \cS}\psi(\ts)\psi(\ts)^\top$. 
We construct our estimator of $M^*$ as:
\begin{align}
\label{eqn:m-estimator}
M_{n} = [A_{n}]^{-1} \sum_{n'<n, h\le H} \phi_{n',h}\psi^{\top}_{n',h} K_{\psi}^{-1},
\end{align}
where
\[
A_{n} = I + \sum_{n'<n, h\le H} \phi_{n', h}\phi_{n', h}^\top.
\]
Let us explain the intuition of $M_{n}$.
Note that 
\begin{align*}
\EE\Big[\phi_{n, h}\psi_{n,h}^{\top}K_{\psi}^{-1}~|~s_{n,h}, a_{n, h}\Big] &= \sum_{\ts}  \phi_{n, h}P(\ts~|~s_{n,h}, a_{n, h})\psi(\ts)^\top K_{\psi}^{-1}
\\
&= 
\sum_{\ts}\phi_{n, h}
\phi_{n, h}^\top M^*\psi(\ts)\psi(\ts)^\top K_{\psi}^{-1}\\
&=\phi_{n, h}\phi_{n, h}^\top M^*.
\end{align*}
Therefore $M_{n}$ is the solution to the following ridge regression problem:
\begin{align}
M_{n} = \arg\min_{M} \sum_{n'<n, h\le H}\Big\| \psi^{\top}_{n',h} K_{\psi}^{-1} - \phi_{n', h}^\top M\Big\|_2^2 + \|M\|_F^2.
\end{align}

\paragraph{Upper confidence RL using a matrix ball.}
In online RL, a critical step is to estimate future value of the current state and action use dynamic programming.
To better balance exploitation and exploration, we use a matrix ball to construct optimistic value function estimator. At episode $n$:
\begin{align}
\label{eqn:compute q}
\forall (s,a) \in \cS\times \cA&:\quad 
Q_{n,H+1}(s,a) = 0 \quad\text{and}\quad \nonumber\\
\forall h\in [H]&: \quad
Q_{n,h}(s,a) = r(s,a) + \max_{M\in B_{n}} \phi(s,a)^\top M \Psi^\top V_{n,h+1}
\end{align}
where \[
V_{n,h}(s) = \Pi_{[0,H]}\big[\max_aQ_{n,h}(s,a)\big]\quad\quad \forall s,a,n,h.
\] 
Here the matrix ball $B_n=B_n^{(1)}$ is constructed as
\begin{align}
\label{eqn:confidence-ball}
B^{(1)}_{n}:=\Big\{M\in \RR^{d\times d'}:\quad \|(A_{n})^{1/2}(M-M_{n})\|_{2,1}\le \sqrt{d\beta_n}\Big\}
\end{align}
where $\beta_n$ is a parameter to be determined later, and $\|Y\|_{2,1} = \sum_{i}\sqrt{\sum_{j} Y{(i,j)}^2}$. 
At time $(n,h)$, suppose the current state is $s_{n,h}$, we play the optimistic action 
$
a_{n,h} = \arg\max_{a} Q_{n,h}(s_{n,h}, a).
$
The full algorithm is given in Algorithm~\ref{alg:core-rl}.
\begin{algorithm*}[htb!]
	\caption{Upper Confidence Matrix Reinforcement Learning (UC-MatrixRL) \label{alg:core-rl}}\small
	\begin{algorithmic}[1]
		\State 
		\textbf{Input:} An episodic MDP environment $M=(\cS,\cA, P, r, s_0, H)$;\\
		\qquad\quad Features $\phi:\cS\times \cA\rightarrow \RR^{d}$ and $\psi:\cS\rightarrow \RR^{d'}$;\\
		\qquad\quad Trotal number of episodes $N$;
		\State \textbf{Initialize:} $A_{1}\gets I\in \RR^{d\times d}$, $M_{1}\gets 0\in \RR^{d\times d'}$;
		\For{episode $n=1, 2, \ldots, N$}
		\State Let $\{Q_{n,h}\}$ be given  in \eqref{eqn:compute q} using $M_n,\beta_n$;
		\For{stage $h=1, 2, \ldots, H$}
		\State Let the current state be $s_{n,h}$;
		\State  Play action $a_{n,h} = \arg\max_{a\in \cA} Q_{n,h}(s_{n,h}, a)$
		\State Record the next state $s_{n, h+1}$;
		\EndFor
		\State $A_{n+1}\gets A_{n} + \sum_{h\le H}\phi_{n,h}\phi_{n,h}^\top$;
		\State Compute $M_{n+1}$ using \eqref{eqn:m-estimator};
		\EndFor
	\end{algorithmic}
\end{algorithm*}

\subsection{Regret Bounds for MatrixRL}
Let $\Psi = [\psi(s_1),\psi(s_2), \ldots, \psi(s_{|\cS|})]^\top\in \RR^{\cS\times d'}$ be the matrix of all $\psi$ features.
We first introduce some regularity conditions of the features space. 
\begin{assumption}[Feature Regularity]
	\label{assump:linrl-regularity}
	Let $C_{M}, C_{\phi}, C_{\psi}$ and $C_{\psi}'$ be positive parameters. 
	\begin{enumerate}
		\item $\|M^*\|_F^2\le C_M \cdot d$;
		\item $\forall~(s,a)\in \cS\times\cA:\quad \|\phi(s,a)\|_{2}^2\le C_{\phi} d$;
		\item $ \forall v\in \RR^{\cS}:\quad \|\Psi^\top v\|_{\infty}\le C_{\psi} \|v\|_{\infty}, \quad \text{and}\quad \|\Psi K_{\psi}^{-1}\|_{2,\infty}\le C_{\psi}'$.\footnote{Here $\|Y\|_{2,\infty}:=\max_{i}\sqrt{\sum_{j}Y^2(i,j)}$ is the operator $2\rightarrow\infty$ norm.}
	\end{enumerate}
\end{assumption}

With these conditions we are ready to provide the regret bound.
\begin{theorem}
	\label{thm:regret-bound}
	Suppose Assumption~\ref{assump:linear transition core} and Assumption~\ref{assump:linrl-regularity} hold.
	 Then after $T=NH$ steps, Algorithm~\ref{alg:core-rl} achieves regret bound:
	\[
	\reg(T)\le O\Big[\sqrt{C_{\psi}(C_M + {C_{\psi}'}^2)}\cdot \ln(C_{\phi}T)\Big]\cdot H^2\cdot \sqrt{d^3 T},
	\]
	if we let
	$
	\beta_k = c\cdot (C_{M} + {C_{\psi}'}^2)\cdot\ln(NHC_{\phi})\cdot d,
	$
	for some  absolute constant $c>0$.
\end{theorem}
 Consider the case where $C_M, C_{\phi}, C_{\psi}$ and $C_{\psi}'$ are absolute constants. For example, we may let $\Psi =I$ and let $\Phi$  be orthonormal functions over $\mathcal{S}$ (meaning that $\sum_{s\in\mathcal{S}}\phi(s)\phi(s)^T=I$). Then Assumption \ref{assump:linrl-regularity} automatically holds with 
$C_M = C_{\psi}=C_{\psi}’=1$. In this case, our regret bound is simply ${O}(d^{3/2}H^2\log(T)\sqrt{T})$.
The $d$-dependence in such a regret bound is consistent with the regret of the $\ell_1$-ball algorithm for linear bandit \cite{dani2008stochastic}.

Further, if the feature space $\Psi$ admits a tighter bound for value function in this space, we can slightly modify our algorithm to achieve sharper regret bound. To do this, we need to slightly change our Assumption~\ref{assump:linrl-regularity} to Assumption~\ref{assump:linrl-regularity-1}.
\begin{customassumption}{2$'$}[Stronger Feature Regularity]
	\label{assump:linrl-regularity-1}
	Let $C_{M}, C_{\phi}, C_{\psi}$ and $C_{\psi}'$ be positive parameters. 
	\begin{enumerate}
		\item $\|M^*\|_F^2\le C_M \cdot d$;
		\item $\forall~(s,a)\in \cS\times\cA:\quad \|\phi(s,a)\|_{2}^2\le C_{\phi} d$;
		\item $ \forall v\in \RR^{\cS}:\quad \|\Psi^\top v\|_{2}\le C_{\psi} \|v\|_{\infty}, \quad \text{and}\quad \|\Psi K_{\psi}^{-1}\|_{2, \infty}\le C_{\psi}'$.
	\end{enumerate}
\end{customassumption}
We modify the algorithm slightly by using a Frobenious-norm matrix ball instead of the 2-1 norm and computing sharper confidence bounds. Let $B_n=B_n^{(2)}$ in \eqref{eqn:compute q}, where
\begin{align}
\label{eqn:confidence-ball-1}
B^{(2)}_{n}:=\Big\{M\in \RR^{d\times d'}:\quad \|(A_{n})^{1/2}(M-M_{n})\|_{F}\le \sqrt{\beta_n}\Big\}
\end{align}
Then a sharper regret bound can be established.
\begin{theorem}
	\label{thm:regret-bound-1}
	Suppose Assumption~\ref{assump:linear transition core} and Assumption~\ref{assump:linrl-regularity-1} hold. Then  after $T=NH$ steps, Algorithm~\ref{alg:core-rl}, with $B_n=B_n^{(2)}$ applied in \eqref{eqn:compute q}, achieves regret
	\[
	\reg(T)\le O\Big[\sqrt{C_{\psi}(C_M + {C_{\psi}'}^2)}\cdot \ln(C_{\phi}T)\Big]\cdot dH^2 \cdot \sqrt{T},
	\]
		provided 
	$
	\beta_n = c\cdot (C_{M} + {C_{\psi}'}^2)\cdot\ln(NHC_{\phi})\cdot d,
	$
	for some  absolute constant $c>0$.
\end{theorem}

The only stronger condition needed by Assumption~\ref{assump:linrl-regularity-1} is $\|\Psi^\top v\|_{2}\le C_{\psi} \|v\|_{\infty}$. It can be satisfied if $\Psi$ is a set of sparse features, or if $\Psi$ is a set of highly concentrated features. 

We remark that in Theorem~\ref{thm:regret-bound} and Theorem~\ref{thm:regret-bound-1}, we need to know the value $N$ in $\beta_n$ before the algorithm runs. 
In the case when $N$ is unknown, one can use the doubling trick to learn $N$ adaptively: first we run the algorithm by picking $N=2$, then for $N=4, 8,\ldots, 2^i$ until the true $N$ is reached. It is standard knowledge that this trick increase the overall regret by only a constant factor (e.g. \cite{besson2018doubling}). 

\textbf{Proof Sketch.}
The proof consists of two parts. We show that when the core matrix $M^*$ belongs to the sequence of constructed balls $\{B_n\}$, the estimated Q-functions provide optimistic estimates of the optimal values, therefore the algorithm's regret can be bounded using the sum of confidence bounds on the sample path. The second part constructs a martingale difference sequence by decomposing a matrix into an iterative sum and uses a probabilistic concentration argument to show that the ``good" event happens with sufficiently high probability.
Full proofs of Theorems~\ref{thm:regret-bound} and \ref{thm:regret-bound-1} are deferred to the appendix. 


\textbf{Near optimality of regret bounds.} The regret bound in Theorem~\ref{thm:regret-bound-1} matches the optimal regret bound $\wt{O}(d\sqrt{T})$ for linear bandit \cite{dani2008stochastic}.
In fact, linear bandit is a special case of RL: the planning horizon $H$ is $1$. 
Therefore our bound is nearly optimal in $d$ and $T$.



\textbf{Implementation.} 
Algorithm 1 can be implemented easily in space $O(d^2)$. 
When implementing Step 6 using \eqref{eqn:compute q}, we do not need to compute the entire $Q$ function as the algorithm only queries the $Q$-values at particular encountered state-action pairs.
For the computation of $\Psi^\top  V_{k,h+1}$, we can apply random sampling over the columns of $\Psi^\top$ to accelerate the computation (see e.g. \cite{drineas2016randnla} for more details).
We can also apply the random sampling method to compute  the matrix $K_{\psi}$ and $K_{\psi}^{-1}$ approximately.

\textbf{Closed-form confidence bounds.} 
Equation~\eqref{eqn:compute q} requires maximization over a matrix ball. However, it is not necessary to solve this maximization problem explicitly. The algorithm only requires an optimistic Q value.
In fact, we can use a closed-form \emph{confidence bound} instead of searching for the optimal $M$ in the \emph{confidence ball}.
It can be verified that Theorem 1 still holds (by following the same proofs of the theorem) if we replace the second equation of \eqref{eqn:compute q} as the following equation (see the proof of Theorem 2)
\begin{align}
\label{eqn:compute q-1}
\forall h\in [H]&: \quad
Q_{n,h}(s,a) = r(s,a) +  \phi(s,a)^\top M_{n} \Psi^\top V_{n,h+1} + 2C_{\psi}H\sqrt{d\beta_{n}}\cdot w_{n,h},
\end{align}
where
$
w_{n,h}:=\sqrt{\phi_{n,h}^\top (A_{n})^{-1}\phi_{n,h}}.
$
Similarly, Theorem~\ref{thm:regret-bound-1} still holds if we replace the the second equation of \ref{eqn:compute q} with 
\begin{align}
\label{eqn:compute q-2}
\forall h\in [H]&: \quad
Q_{n,h}(s,a) = r(s,a) +  \phi(s,a)^\top M_{n} \Psi^\top V_{n,h+1} + 2C_{\psi}\sqrt{\beta_{n}}\cdot w_{n,h}.
\end{align}
Equations \eqref{eqn:compute q-1} and \eqref{eqn:compute q-2} can be computed easily. They can be viewed as the ``dualization'' of \eqref{eqn:compute q}. 

\section{RL Exploration in Kernel Space}
\label{sec:rl-kernel}

\def\bK{\mathbf{K}}
\def\bk{\mathbf{k}}

In this section, we transform MatrixRL to work with kernels instead of explicitly given features. Suppose we are given two reproducing kernel Hilbert spaces $\cH_{\phi}, \cH_{\psi}$ with kernel functions $k_{\phi}:(\cS\times \cA)\times (\cS\times \cA)\rightarrow \RR$ and $k_{\psi}:\cS\times \cS\rightarrow \RR$, respectively. There exists implicit features  $\phi,\psi$ such that $k_{\phi}(x,y) = \phi(x)^\top\phi(y)$ and $k_{\psi}(x,y) = \psi(x)^\top\psi(y)$, but the learning algorithm can access the kernel functions only.


\subsection{Kernelization of MatrixRL}
The high level idea of Kernelized MatrixRL is to represent all the features used in Algorithm~\ref{alg:core-rl} with their corresponding kernel representations.
We first introduce some notations.
For episode $n$ and $T=nH$, we denote $\bK_{\phi_{n}}\in \RR^{T\times T}$ and $\bK_{\psi_{n}}\in \RR^{T\times T}$ as the Gram matrix, respectively, i.e., for all $t_1=(n_1, h_1), t_2=(n_2, h_2)\in [n]\times [H]$,
\[
\bK_{\phi_{n}}[t_1, t_2] = k_{\phi}[(s_{t_1}, a_{t_1}), (s_{t_2}, a_{t_2})],
\qquad
\bK_{\psi_{n}}[t_1, t_2] = k_{\psi}(\wt{s}_{t_1}, \wt{s}_{t_2}),
\]
where $\wt{s}_t:=s_{t+1}$.
We denote $\wb{\bK}_{\psi_{n}}\in \RR^{T\times |\cS|}$ and $\bk_{\Phi_{n}, s,a} \in \RR^{T}$
by
 \[
\wb{\bK}_{\psi_{n}}[t_1, s] = k_{\psi}(\wt{s}_{t_1}, s),
\qquad
\bk_{\Phi_{n}, s,a}[t] = k_{\phi}[(s_{t}, a_{t}), (s,a)].
\]
We are now ready to kernelize Algorithm~\ref{alg:core-rl}.
The full algorithm of Kernelized MatrixRL is given in Algorithm~\ref{alg:core-rl-kernel}.
Note that the new Q function estimator \eqref{eqn:compute q-kernel} is the dualization form of \eqref{eqn:compute q}.
Therefore
Algorithm~\ref{alg:core-rl-kernel} is more general but essentially 
equivalent to Algorithm~\ref{alg:core-rl} if we let $k_{\phi}(x,y) := \phi(x)^\top\phi(y)$ and $k_{\psi}(x,y) := \psi(x)^\top\psi(y)$.
See Section~\ref{sec:derive-kernel} for the proof.

\begin{algorithm*}[htb!]
	\caption{KernelMatrixRL: Reinforcement Learning with Kernels \label{alg:core-rl-kernel}}\small
	\begin{algorithmic}[1]
		\State 
		\textbf{Input:} An episodic MDP environment $M=(\cS,\cA, P, s_0, r, H)$, kernel functions $k_{\phi},k_{\psi}$;\\
		\qquad\quad Total number of episodes $N$;
		\State \textbf{Initialize:} empty reply buffer $\cB=\{\}$;
		\For{episode $n=1, 2, \ldots, N$}
		\State For $(s,a)\in \cS\times \cA$, let
		\vspace{-3mm}
		\begin{align*}
		w_{n}(s,a) &:= \sqrt{k_{\phi}[(s,a),(s,a)]
			-\bk_{\Phi_{n-1},s,a}^\top (I+\bK_{\Phi_{n-1}})^{-1}\bk_{\Phi_{n-1},s,a}};\\
		x_{n}(s,a)&:= \bk_{\Phi_{n-1},s,a}^\top (I+\bK_{\Phi_{n-1}})^{-1} \bK_{\Psi_{n-1}} (\wb{\bK}_{\Psi_{n-1}}\wb{\bK}_{\Psi_{n-1}}^\top)^{-1} \wb{\bK}_{\Psi_{n}};
		\end{align*}
		\vspace{-3mm}
		\State Let $\{Q_{n,h}\}$ be defined as follows: 
		\begin{align}
		\label{eqn:compute q-kernel}
		\forall (s,a) \in \cS\times \cA&:\quad 
		Q_{n,H+1}(s,a) := 0 \quad\text{and}\quad \nonumber\\
		\forall h\in [H]&: \quad
		Q_{n,h}(s,a) := r(s,a) + x_{n}(s,a)^\top V_{n,h+1}
		+ \eta_{n} w_{n}(s,a),
		\end{align}
		where \[
		V_{n,h}(s) = \Pi_{[0,H]}\big[\max_aQ_{n,h}(s,a)\big]\quad\quad \forall s,a,n,h;
		\] 
		and $\eta_n$ is a parameter to be determined;
		\For{stage $h=1, 2, \ldots, H$}
		\State Let the current state be $s_{n,h}$;
		\State  Play action $a_{n,h} = \arg\max_{a\in \cA} Q_{n,h}(s_{n,h}, a)$;
		\State Record the next state $s_{n, h+1}$: $\cB\gets \cB\cup\{(s_{n, h}, a_{n,h}, s_{n, h+1})\}$;
		\EndFor
		\EndFor
	\end{algorithmic}
\end{algorithm*}

\subsection{Regret Bound for Kernelized MatrixRL}
\vspace{-3mm}
We define the \emph{effective dimension} of the kernel space $\cH_{\phi}$ as
\[
\wt{d} = \sup_{t\le NH} \sup_{X\subset \cS\times \cA, |X| =t} \frac{\log \det[I + \bK_{X}]}{\log (1+t)},
\]
where $X=\{x_j\}_{j\in[t]}$, $\bK_{X}\in \RR^{t\times t}$ with $\bK_{X}[i,j] = k_{\phi}(x_i, x_j)$ is the Gram matrix over data set $X$.
Note that $\wt{d}$ captures the effective dimension of the space spanned by the features of state-action pairs. 
Consider the case when $\phi$ are $d$-dimensional  unit vectors, then $\cH_{\phi}$ has dimension at most $d$. It can be verified that $\wt{d}\le d$.
A similar notation of effective dimension was introduced by \cite{Valko} for analyzing kernelized contextual bandit.

Further, we need regularity assumptions for the kernel space. 
\begin{assumption}
\label{assum:kernel}
Let $\cH_{\psi}$ be generated by orthonormal basis on $\mathcal{S}$, i.e., there exists $\psi$ such that $\sum_{s\in\mathcal{S}}\psi(s)\psi(s)^\top = I$ and $k_{\psi}(s,s')=\psi(s)^\top\psi(s')$. There exists a constant $C_{\psi}$ such that 
\[
\forall v\in \cH_{\psi}:\quad \|v\|_{\cH_{\psi}}\le C_{\psi}\|v\|_{\infty},
\]
where $\|\cdot\|_{\cH_{\psi}}$ denotes the Hilbert space norm. 
\end{assumption}
The formal guarantee of Kernelized MatrixRL is presented as follows.
\begin{theorem}
	\label{thm:regret-bound-kernel}
	Suppose 
	the probability transition kernel $P(\cdot\mid \cdot)$ belongs to the product Hilbert spaces, i.e., $P\in \cH_{\phi}\times \cH_{\psi}$. Let Assumption \ref{assum:kernel} hold.
	Then after $T=NH$ time steps,  the regret of Algorithm~\ref{alg:core-rl-kernel} satisfies
	\[
	\reg(T)\le 
	O\Big(C_{\psi}\cdot\|P\|_{\cH_{\phi}\times \cH_{\psi}}
	\cdot \log(T)\cdot \wt{d} \cdot H^2\cdot \sqrt{T}\Big)
	\] 
	provided $\eta_n = 2C_{\psi} H\sqrt{\beta_n}$ and 
	$
	\beta_n = \Theta( \|P\|_{\cH_{\phi}\times \cH_{\psi}} 
	\cdot\ln(NH)\cdot \wt{d})$. 
\end{theorem}
Note that in Assumption~\ref{assum:kernel}, we can additionally relax the assumption on the orthogonality of $\psi$.
Similar regret bound can be proved with Assumption~\ref{assump:linrl-regularity-1}.
The proof of Theorem \ref{thm:regret-bound-kernel} is very similar to that of Theorem~\ref{thm:regret-bound-1}. 
Although Kernelized MatrixRL does not access the features, the proof is based on the underlying features and the equivalence between kernel representation and feature representation. We postpone it to Section~\ref{sec:derive-kernel}.

\textbf{Remark.}
Similar as MatrixRL, Kernelized MatrixRL can be generalized to deal with unknown reward function  by using the Kernelized Bandit  \cite{Valko}.
Again, since linear bandit problems are  special cases of kernel RL with $H=1$, our results match the linear bandit  bound on $\wt{d}$ and $\sqrt{T}$. The computation time of Kernelized MatrixRL scales with time $T$ as $T^2$ (by applying randomized algorithms, e.g. \cite{dani2008stochastic}, in dealing with $\wb{\bK}_{\Psi}$ matrices), still polynomial in $T$. 
We can apply the random features or sketching techniques for kernel to additionally accelerate the computation (e.g. \cite{rahimi2008random, yang2017randomized}). 


\section{Summary}
This paper provided the algorithm MatrixRL for episodic reinforcement learning in high dimensions. It also provides the first regret bounds that are near-optimal in time $T$ and feature dimension $d$ and polynomial in the planning horizon $H$. 
MatrixRL uses given features (or kernels) to estimate a core transition matrix and its confidence ball, which is used to compute optimistic Q-functions for balancing the exploitation-exploration tradeoff. 
We prove that the regret of MatrixRL is bounded by ${O}\big(H^2d\log T\sqrt{T}\big)$ where $d$ is the number of features, provided that the feature space satisfies some regularity conditions. 
MatrixRL has an equivalent kernel version, which does not require explicit features. The kernelized MatrixRL satisfies a regret bound ${O}\big(H^2\wt{d}\log T\sqrt{T}\big)$, where $\wt{d}$ is the effective dimension of the kernel space.
For future work, it remains open if the regularity condition can be relaxed and if there is a more efficient way for constructing confidence balls in order to further reduce the regret. 

\clearpage
\bibliographystyle{apalike}
\bibliography{reference}

\clearpage
\appendix

\section{Analysis and Proofs}
In this section we will focus on proving Theorem~\ref{thm:regret-bound}.
In the proof we will also establish  all the necessary analytical tools for proving Theorem~\ref{thm:regret-bound-1} and Theorem~\ref{thm:regret-bound-kernel}.
We provide the proofs of the last two theorems at the end of this section. 

The proof of Theorem~\ref{thm:regret-bound} consists of two steps: (a) We first show that  if the true transition core $M^*$ is always in the confidence ball $B_{n}$, defined in Equation~\ref{eqn:confidence-ball}, we can then achieve the desired regret bound;
(b) We then show that with high probability, the event required by (a) happens.
We formalize the event required by step (a) as follows.
\begin{definition}[Good Estimator Event]
\label{assum:always-good-estimator}
For all $n\in [N]$, we denote $E_n = 1$ if $M^*\in B_{n'}$ for all $n'
\in [n]$ and otherwise $E_n=0$.
\end{definition}
Note that $E_n$ is completely determined by the game history up to episode $n$. In the next section, we show (a).
\subsection{Regret Under Good Event}
To better investigate the regret formulation \eqref{eqn:regret-def}, we rewrite it according to Algorithm~\ref{alg:core-rl}.
Note that conditioning on the history before episode $n$,  the algorithm plays a fixed policy $\pi_{n}$ for episode $n$.
Therefore, we have
\begin{align}
\label{eqn:new-regret}
\reg(NH)= \EE\Big[\sum_{n=1}^{N} \Big(V^*_1(s_0) - V^{\pi_{n}}_1(s_0)\Big)\Big]
=:\sum_{n=1}^N\EE[R(n)], 
\end{align}
where $R(n)=V^*_1(s_0) - V^{\pi_{n}}_1(s_0)$.
We now show that the algorithm always plays an optimistic action (an action with value estimated greater than the optimal value of the state).
\begin{lemma}[Optimism]
Suppose 
for $n\in [N]$, we have the good estimator event, $E_n=1$, happens. Then for $ h\in [H]$ and $(s,a)\in \cS\times \cA$, we have
\[
Q^*_h(s,a)\le Q_{n,h}(s,a).
\]
\end{lemma}
\begin{proof}
We prove the lemma by induction on $h$.
It is vacuously true for the case $h=H$ since $Q_{n,H}(s,a)=r(s,a) = Q^*_H(s,a)$.
Suppose the lemma holds for some $1<h'\le H$.
We then have
\[
\forall s\in \cS:
V_{n,h'}(s) =  \Pi_{[0,H]}[\max_aQ_{n,h}(s,a)] \ge V_{h'}^*(s).
\]
We now consider $h=h'-1$.
Note that
\begin{align*}
\forall (s,a)\in \cS\times \cA:\quad 
Q_{n,h}(s,a) &= r(s, a)  + \max_{M\in B_n} \phi(s,a)^\top M \Psi^\top V_{n,h'}\\
&\ge  r(s, a)  + \phi(s,a)^\top M^* \Psi^\top V_{n,h'}\\
& = r(s, a)  + P(\cdot|s,a)^\top V_{n,h'}\\
&\ge r(s, a)  + P(\cdot|s,a)^\top V^*_{h'}
= Q_{h}^*(s,a).
\end{align*}
\end{proof}
Next we show that the confidence ball $B_{n}$ actually gives a strong upper bound for the estimation error: the estimation error is ``along'' the direction of the exploration.
\begin{lemma}
	\label{lemma:ball upper bound}
For any $M\in B_{n}$ we have
\[
\|\phi_{s,a}^\top(M-M_{n})\|_1
\le \sqrt{d\beta_{n} \phi(s,a)^\top (A_{n})^{-1} \phi(s,a)}.
\]
\end{lemma}
\begin{proof}
\begin{align*}
\|\phi_{s,a}^\top(M-M_{n})\|_1
&=  \|\phi_{s,a}^\top(A_{n})^{-1/2}(A_{n})^{1/2}(M-M_{n})\|_1\\
&\le \|\phi_{s,a}^\top(A_{n})^{-1/2}\|_2\cdot
\|(A_{n})^{1/2}(M-M_{n})\|_{2,1}\\
&\le \sqrt{d\beta_n}\cdot \|\phi_{s,a}^\top(A_{n})^{-1/2}\|_2
\end{align*}
as desired.
\end{proof}
 
Next we show that the value iteration per-step  does not introduce too much error.
\begin{lemma}
	\label{lemma:induction}
Suppose for $n\in [N]$, $E_n=1$. Then for $h\in [H]$, we have
\[
Q_{n,h}(s_{n,h}, a_{n,h}) - 
\Big[r(s_{n,h}, a_{n,h}) + P(\cdot|s_{n,h}, a_{n,h})^\top V_{n,h+1}\Big]
\le 2C_{\psi}H\sqrt{d\beta_{n}}\cdot w_{n,h}
\]
where
\begin{align}
w_{n,h}:=\sqrt{\phi_{n,h}^\top (A_{n})^{-1}\phi_{n,h}}.
\end{align}
\end{lemma} 
\begin{proof}
Let
\[
\wt{M} =\arg\max_{M\in B_{n}} \phi_{n,h}^\top M \Psi^\top 
V_{n,h+1}.
\]
We then have
\begin{align*}
Q_{n,h}(s_{n,h}, a_{n,h}) &- 
\Big[r(s_{n,h}, a_{n,h}) + P(\cdot|s_{n,h}, a_{n,h})^\top V_{n,h+1}\Big]\\
&=  \phi_{n,h}^\top (\wt{M} - M^*) \Psi^\top 
V_{n,h+1}\\
&\le \| \phi_{n,h}^\top (\wt{M} - M^*)\|_1\cdot 
\|\Psi^\top 
V_{n,h+1}\|_{\infty}\\
&\le  C_{\psi} H\| \phi_{n,h}^\top (\wt{M} - M^*)\|_1\\
&\le C_{\psi}H ( \| \phi_{n,h}^\top (\wt{M} - M_{n})\|_1 +  \| \phi_{n,h}^\top ( M_{n} - M^*)\|_1)\\
&\le 2C_{\psi}H\sqrt{d\beta_{n}\cdot \phi_{n,h}^\top (A_{n})^{-1}\phi_{n,h}}.
\end{align*}
\end{proof}
We are now ready to show the regret bound.
\begin{lemma}
Suppose Assumption~\ref{assum:always-good-estimator} holds, $1\le \beta_{1}\le \beta_{2}\le\ldots \beta_{N}$, then,
\[
\reg(NH)
\le 2C_{\psi}H\sqrt{d\beta_{N}}\cdot\EE\bigg[\sum_{n=1}^N\sum_{h=1}^H\sqrt{\min(1,w_{n,h}^2)}
\bigg] + \sum_{n=1}^NH\PP[E_n\neq 1]
\]
\end{lemma}
\begin{proof}
Recall that
\[
\reg(NH)= \EE\Big[\sum_{n=1}^{N} \Big(V^*_1(s_0) - V^{\pi_{n}}_1(s_0)\Big)\Big]
= \sum_{n=1}^N\EE\big[R(n)\big].
\]
Consider $R(n)$ for a fixed $n$.
Denote $\cF_{n,h}$ as the filtration of fixing the history up to time $(n,h)$ (i.e., fixing $s_{1,1}, a_{1,1},s_{1,2}, a_{1,2},\ldots, s_{n,h}, a_{n,h}$ but not $s_{n,h+1}$). 
Since if $E_n=0$, we can always bound $R(n)\le H$.
We then have
\begin{align*}
\EE[R(n)]=\EE[R(n)E_n + (1-E_n)R(n)]\le \EE[R(n)E_n] + H\Pr[E_n = 0].
\end{align*}
We then denote
\[
\EE[R(n)E_n] = \EE[\EE[R(n)E_n|\cF_{n,1}]].
\]
Expanding the right hand side, we have,
\begin{align*}
\EE\big[R(n)E_n~\big|~\cF_{n,1}\big]
&= (V_1^*(s_{n, 1}) - V_1^{\pi_{n}}(s_{n,1}))\cdot E_n\\
&\le (Q_{n,h}(s_{n, 1}, a_{n, 1}) - V^{\pi_{n}}_1(s_{n,1}))\cdot E_n\\
&\le  2C_{\psi}H\sqrt{d\beta_{n}}\cdot w_{n,1} + 
P(\cdot~\big|~s_{n, 1}, a_{n, 1})^\top (V_{n,2} - V^{\pi_{n}}_2) \cdot E_n \qquad\qquad\text{by Lemma~\ref{lemma:induction}}\\
&\le\EE\Big[ 2C_{\psi}H\sqrt{d\beta_{n}}\cdot w_{n,1} + 
(V_{n,2}(s_{n,2}) - V^{\pi_{n}}_2(s_{n,2}))\cdot E_n~\big|~\cF_{n,1}\Big]\\
&\le \EE\Big[2C_{\psi}\sqrt{d\beta_{n}}H\cdot( w_{n,1} 
+  w_{n,2})
 + \EE\Big[(V_{n,3}(s_{n,3}) - V^{\pi_{n}}_3(s_{n,3}))\cdot E_n~\big|~\cF_{n,2}\Big]~\big|~\cF_{n,1}\Big]\le \ldots\\
&\le 2C_{\psi}H\cdot\EE\Big[\sqrt{d \beta_{n}}\cdot\sum_{h=1}^{H} w_{n, h}~\Big| ~ \cF_{n,1}\Big].
\end{align*}
Moreover, we immediately have
\[
 V_1^*(s_{n, 1}) - V^{\pi_{n}}_1(s_{n,1})
 \le H.
\]
Therefore
\begin{align*}
\reg(NH)&\le 
\sum_{n=1}^N \EE\Big[\min\Big(H, \sum_{h=1}^H 2C_{\psi}H\cdot\sqrt{d\beta_{n} \cdot w_{n,h}^2}~\Big)\Big] + \sum_{n=1}^{N}H\Pr[E_n = 0]\\
&\le 2C_{\psi}H\sqrt{d\beta_{N}}\cdot\EE\bigg[\sum_{n=1}^N\sum_{h=1}^H\sqrt{\min(1,w_{n,h}^2)} 
\bigg] + \sum_{n=1}^{N}\Pr[E_n = 0].
\end{align*}
This completes the proof.
\end{proof}
It remains to bound 
\[
\sum_{n=1}^N\sum_{h=1}^H\sqrt{\min(1,w_{n,h}^2)}
\le \sqrt{HN\cdot\sum_{n=1}^N\sum_{h=1}^H\min(1,w_{n,h}^2)}.
\]
We provide the following lemma.
\begin{lemma}
	\label{lemma:sum-confidence}
\[
\sum_{n=1}^N\sum_{h=1}^H\min(1,w_{n,h}^2)
\le  2H\ln\det(A_{N+1})
\le 2Hd\ln(NHC_{\phi} + 1).
\]
\end{lemma}
The proof of this lemma is rather technical and requires some new notations, we postpone it to Section~\ref{sec:full-proof-regret}. 
We are now ready to state the regret bound.
\begin{lemma}
	\label{lemma:conitioning regret}
Suppose $1\le \beta_{1}\le \beta_{2}\le\ldots \beta_{N}$ and $C_{\psi}\ge 1$, then
\[
\reg(T)\le 
2C_{\psi}H\sqrt{d\beta_{N}} \cdot \sqrt{HN\cdot  2Hd\cdot \ln[NHC_{\phi} + 1]}
+ \sum_{n=1}^N H\PP[E_n=0]
.
\]
\end{lemma}

\subsection{Concentration}
In this section, we show that $E_n=1$ holds with high probability through out the online learning process.
We begin with some notations and axillary random variables.
For $(n', h'), (n,h)\in [N]\times [H]$, we denote $(n',h') < (n,h)$ if $(n', h')$ is lexicographically before $(n,h)$, i.e., either $n'<n$ or $n=n'$ but $h'< h$.
For each $(n,h)\in [N]\times [H]$, we denote
\[
A_{n,h} = A_{n} + \sum_{h'<h}\phi_{n,h'}\phi_{n,h'}^\top \quad\text{and}\quad
M_{n,h} = A_{n,h}^{-1}\sum_{(n', h')< (n,h)} \phi_{n', h'}\psi_{n',h'} K_{\psi}^{-1}.
\]
Notice that $ A_{n} = A_{n,1}$ and $M_{n}=M_{n,1}$.
Moreover, we denote
\[
\tw_{n,h}^2 := \phi_{n,h}^\top A_{n,h}^{-1} \phi_{n,h}.
\]
Similar to Lemma~9 of \cite{dani2008stochastic}, we have the following lemma.
\begin{lemma}
\label{lem:detakh}
$
\sum_{(n',h')<[n,h]}\min(1, \tw_{n',h'}^2)
\le 2 \ln\det(A_{n,h}) \le 2d\ln[(n-1)HC_{\phi} + (h-1)C_{\phi} + 1]
$
\end{lemma}
\begin{proof}
Firstly, we have $\min(1, \tw_{n',h'}^2)\le 2\ln (1+ \tw_{n',h'}^2)$.
Next, consider $A_{n,h}$.
We have
\begin{align*}
\det(A_{(n,h)+1}) &= \det(A_{n,h} + \phi_{n, h} \phi_{n, h}^\top)\\
&=\det(A_{n,h})\det\Big(I + \sqrt{A_{n,h}^{-1}}\phi_{n, h} \phi_{n, h}^\top\sqrt{A_{n,h}^{-1}} \Big)\\
&= \det(A_{n,h})(1+\tw_{n, h}^2).
\end{align*}
By induction, we have
\[
\det(A_{(n,h)+1}) = \prod_{(n', h')\le (n,h)}(1+\tw_{n', h'}^2) \det(A_{1,1}).
\]
Note that $\det(A_{1,1}) = 1$, which proves the first inequality.

Next we consider the second inequality. 
Since $A_{n,h}$ is positive definite (PD), instead of considering the determinant directly, we consider the trace of $A_{n,h}$.
\begin{align*}
\tr(A_{n,h})
&= \tr\Big[I + \sum_{(n', h')< (n,h)}\phi_{n', h'}\phi_{n', h'}^\top \Big]\le d + \sum_{(n', h')< (n,h)} \|\phi_{n', h'}\|_2^2\\
&\le (n-1)HC_{\phi}d + (h-1)C_{\phi}d + d.
\end{align*}
where we use the fact that $\|\phi_{n', h}\|_2^2\le C_{\phi}d$.
Since $A_{n,h}$ is PD, for the worst case we have
\[
\det(A_{n,h}) \le (\tr(A_{n,h})/d)^d\le [(n-1)HC_{\phi} + (h-1)C_{\phi} + 1]^d.
\]
\end{proof}
We also let $\beta_{n,h} = \beta_{n}$ for all $n,h$.
We consider the following random variables.
\[
Z_{n,h} = \tr\big[(M^* - M_{n,h})^\top A_{n,h}(M^* - M_{n,h})\big].
\]
Note that
\[
\|A_{n,h}^{1/2}(M^* - M_{n,h})\|_{2,1}
\le \sqrt{d\cdot Z_{n,h}}.
\]
If we can bound $Z_{n,h}$ for all $n,h$, we can therefore conclude whether $M^*$ is in the ball $B_{n}$.
To prove the that $Z_{n,h}$ is indeed small, we use similar techniques developed in \cite{dani2008stochastic}. 
We denote 
\[
\eta_{n,h} = K_{\psi}^{-1}\psi_{n,h} - (M^*)^{\top}\phi_{n,h}.
\]
Note that 
\[
\|(M^*)^{\top}\phi_{n,h}\|_2 =
\|\EE[K_{\psi}^{-1}\psi_{n,h}~|~\cF_{n,h}]\|_2
\le \EE[\|K_{\psi}^{-1}\psi_{n,h}\|_2~|~\cF_{n,h}]
\le C_{\psi}'.
\]
We then have,
\[
\EE\big[\eta_{n,h}~\big|~\cF_{n,h}\big] = 0 \quad\text{and}\quad 
\|\eta_{n,h}\|_{2}^2\le 2C_{\psi}'.
\]
The next lemma bounds the growth of $Z_{n,h}$.
\begin{lemma}
	\label{lemma:zkh}
For all $(n,h)$, we have\[
Z_{n,h} \le 
\|M^*\|_F^2 + 2 \sum_{(n', h')<(n,h)}\frac{\phi_{n', h'}^\top(M_{n',h'}-M^*)\eta_{n', h'}}{1+\tw_{n', h'}^2}
+ \sum_{(n', h')<(n,h)}
\|\eta_{n', h'}\|_2^2 \frac{\tw_{n', h'}^2}{1+\tw_{n', h'}^2}
\]
\end{lemma}
\begin{proof}
The proof is very similar to Lemma~12 of \cite{dani2008stochastic}.
For completeness, we present the proof here.
We first introduce the following notation.
\[
Y_{n,h} = A_{n,h}(M_{n,h} - M^*).
\]
Let $(n,h)+1 = (n, h+1)$ if $h< H$, or $(n+1, 1)$ if $h=H$.
We then have
\begin{align*}
Z_{n,h} &= \tr\big[Y_{n,h}^\top A_{n,h}^{-1} Y_{n,h}\big]\\
Y_{n,h} &= \sum_{(n', h')<(n,h)} \phi_{n', h'}\eta_{n', h'}^\top  - M^* \\
Y_{(n,h)+1}&= Y_{n,h} +  \phi_{n, h}\eta_{n, h}^\top.
\end{align*}
Now we consider $Z_{(n,h)+1}$, we have
\begin{align*}
Z_{(n,h)+1} &= \tr\big[Y_{(n,h)+1}^{\top} A_{(n,h)+1}^{-1}Y_{(n,h)+1}\big]\\
&= \tr\big[(Y_{n,h} +  \phi_{n, h}\eta_{n, h}^\top)^\top A_{(n,h)+1}^{-1} (Y_{n,h} +  \phi_{n, h}\eta_{n, h}^\top)\big]\\
&=\tr\big[Y_{n,h}^{\top} A_{(n,h)+1}^{-1}Y_{n,h}\big]
+ 2\tr\big[\eta_{n,h}\phi_{n,h}^\top A_{(kn,h)+1}^{-1}Y_{n,h}\big]
+ \|\eta_{n,h}\|_{2}^2\cdot \phi_{n,h}^\top A_{(n,h)+1}^{-1}\phi_{n,h}.
\end{align*}
Applying the matrix inversion lemma to $A_{(n,h)+1}$, we have
\[
A_{(n,h)+1}=(A_{n,h}+\phi_{n,h}\phi_{n,h}^\top)^{-1}
= A_{n,h}^{-1} - \frac{A_{n,h}^{-1}\phi_{n,h}\phi_{n,h}^\top A_{n,h}^{-1}}{1+\tw_{n,h}^2}.
\]
We thus have
\[
\tr\big[Y_{n,h}^{\top} A_{(n,h)+1}^{-1}Y_{n,h}\big]
= \tr\big[Y_{n,h}^{\top} A_{n,h}^{-1}Y_{n,h}\big] - 
\tr\Big[Y_{n,h}^{\top}\frac{A_{n,h}^{-1}\phi_{n,h}\phi_{n,h}^\top A_{n,h}^{-1}}{1+\tw_{n,h}^2} Y_{n,h}\Big]
\le \tr\big[Y_{n,h}^{\top} A_{n,h}^{-1}Y_{n,h}\big]
= Z_{n,h}.
\]
For the second term,
\begin{align*}
2\tr[\eta_{n,h}\phi_{n,h}^\top A_{(n,h)+1}^{-1}Y_{n,h}\big]
&= 2\tr\bigg[\eta_{n,h}\phi_{n,h}^\top A_{(n,h)}^{-1}Y_{n,h}
 - 
 \frac{\eta_{n,h}\phi_{n,h}^\top A_{n,h}^{-1}\phi_{n,h}\phi_{n,h}^\top A_{n,h}^{-1} Y_{n,h}}{1+\tw_{n,h}^2}
 \bigg]\\
&=
2\tr\bigg[\eta_{n,h}\phi_{n,h}^\top(M_{n,h}-M^*)
- 
\frac{\eta_{n,h}\tw_{n,h}^2\phi_{n,h}^\top (M_{n,h}-M^*)}{1+\tw_{n,h}^2}
\bigg]\\
&=
2\tr\bigg[\frac{\eta_{n,h}\phi_{n,h}^\top (M_{n,h}-M^*)}{1+\tw_{n,h}^2}\bigg].
\end{align*}
For the third term, 
\begin{align*}
 \|\eta_{n,h}\|_{2}^2\cdot \phi_{n,h}^\top A_{(n,h)+1}^{-1}\phi_{n,h}
 &= \|\eta_{n,h}\|_{2}^2 \tw_{n,h}^2
  - \|\eta_{n,h}\|_{2}^2 \frac{\tw_{n,h}^4}{1+\tw_{n,h}^2}= \|\eta_{n,h}\|_{2}^2 \frac{\tw_{n,h}^2}{1+\tw_{n,h}^2}.
\end{align*}
Putting these together, we have
\begin{align*}
Z_{(n,h)+1}
&\le Z_{(n,h)}
+ 2\frac{\phi_{n,h}^\top (M_{n,h}-M^*)\eta_{n,h}}{1+\tw_{n,h}^2}
+ \|\eta_{n,h}\|_{2}^2 \frac{\tw_{n,h}^2}{1+\tw_{n,h}^2}\\
&\le Z_{(n,h)-1} + \ldots\\
&\le Z_{1,1} + 
2 \sum_{(n', h')\le(n,h)}\frac{\phi_{n', h'}^\top(M_{n',h'}-M^*)\eta_{n', h'}}{1+\tw_{n', h'}^2}
+ \sum_{(n', h')\le(n,h)}
\|\eta_{n', h'}\|_2^2 \frac{\tw_{n', h'}^2}{1+\tw_{n', h'}^2}.
\end{align*}
Lastly, we check $Z_{1,1}$.
\[
Z_{1,1} = \tr\Big[(M_1 - M^*)^\top (M_1 - M^*)\Big]
= \|M^*\|_F^2.
\]
This completes the proof.
\end{proof}
We now define a martingale difference sequence. In order to upper bound the variance of the random variables, we consider 
\[
E_{k,h} = \mathbb{I}\{Z_{n',h'}\le \beta_{n',h'} \text{ for all } (n',h')\le (k,h)\}
\]
\begin{lemma}
	Let 
	\[
	G_{n,h} = 2E_{n,h}\frac{\phi_{n, h}^\top(M_{n,h}-M^*)\eta_{n, h}}{1+\tw_{n, h}^2}.
	\]
	Then $G_{n,h}$ is a martingale difference sequence with respect to $\cF_{n,h}$.
\end{lemma}
\begin{proof}
	Since $\cF_{n,h}$ determines $M_{n,h}$, $\phi_{n,h}$, $\tw_{n, h}$, and $Z_{n,h}$ (note that in the definition of $Z_{n,h}$, variable $\psi_{n,h}$ is not included), we have
\begin{align*} 
\EE[G_{n,h}~|~ \cF_{n,h}] =  2E_{n,h}\frac{\phi_{n', h}^\top(M_{n,h}-M^*)}{1+\tw_{n, h}^2} \cdot 
\EE[\eta_{n, h}~|~\cF_{n,h}] = 0.
\end{align*}
\end{proof}

We will show that with high probability, the martingale difference sum, $\sum_{(n', h')<(n,h)} G_{(n', h')}$ never grows too large.
\begin{lemma}[Concentration Lemma]
	\label{lemma:contrencetion lemma}
Given $\delta < 1$, we have
\[
\PP\Big[\forall (n,h),\quad \sum_{(n', h')<(n,h)}G_{(n', h')} \le \beta_{n',h'}/2\Big]
\ge 1-\delta.
\]	
\end{lemma}
To prove this lemma, we will apply the Freedman's inequality.
\begin{theorem}[Freedman \cite{Freedman1975}]
Let $X_1, X_2, \ldots, X_T$ be a martingale difference sequence with respect to $Y_1, Y_2, \ldots, Y_T$.
Let $b$ be an uniform upper bound on $X_i$. Let $V$ be the sum of conditional variances,
\[
V = \sum_{t\le T} \var\Big(X_i~\Big|~Y_1, Y_2, \ldots, Y_{i-1}\Big).
\]
Then for every $a, v > 0$, 
\[
\PP\Big[\sum_{i=1}^T X_i \ge a \text{ and }  V \le v\Big]
\le \exp\bigg(\frac{-a^2}{2v + 2ab/3}\bigg).
\]
\end{theorem}
\begin{proof}[Proof of Lemma~\ref{lemma:contrencetion lemma}]
We first show the upper bounds on the step size of the martingale $G_{k,h}$.
\begin{align*}
|G_{n,h}|
&\le 2\|\eta_{n, h}\|_{2} \bigg\|E_{n,h}\frac{\phi_{n, h}^\top(M_{n,h}-M^*)}{1+\tw_{n, h}^2}\bigg\|_{2}\\
&\le 2\|\eta_{n, h}\|_{2}\cdot E_{n,h}\cdot \frac{\|\phi_{n, h}^\top A_{n,h}^{-1/2}\|_{2}\|A_{n,h}(M_{n,h}-M^*)\|_F}{1+\tw_{n, h}^2}\\
&\le 2\|\eta_{n, h}\|_{2}\cdot
\sqrt{\beta_{n,h} }\cdot E_{n,h}\cdot\frac{\tw_{n, h}}{1+\tw_{n, h}^2}\\
&\le  2\|\eta_{n, h}\|_{2}\cdot
\sqrt{\beta_{n,h} }\cdot E_{n,h}\cdot \min(\tw_{n, h}, 1/2) \\
&\le 2C_{\psi}'\sqrt{\beta_{n,h}}.
\end{align*}
Next we bound the conditional variance of $G_{n,h}$.
We have,
\begin{align*}
V_{n,h}&:=\sum_{(n', h')< (n,h)}
\var\big(G_{n', h'}~|~\cF_{(n',h')-1}\big)\\
&\le \sum_{(n', h')<(k,h)}
 4\|\eta_{n, h}\|_{2}^2\cdot  E_{n,h}\cdot 
\beta_{n,h} \cdot \min(\tw_{n, h}, 1/2)^2\\
&\le \sum_{(n', h')<(n,h)}
4\|\eta_{n, h}\|_{2}^2\cdot   E_{n,h}\cdot 
\beta_{n,h} \cdot \min(\tw_{n, h}^2, 1)
\end{align*}
By Lemma~\ref{lem:detakh}, we have
\[
V_{n,h} 
\le 8 {C'}_{\psi}^{2}\cdot \beta_{n,h}\cdot \ln\det(A_{n,h}) =: v_{n,h}.
\]
Note that
\[
v_{n,h}\le 8 {C'}_{\psi}^{2}\cdot \beta_{n,h} \cdot d \cdot \ln[(n-1)HC_{\phi} + (h-1)C_{\phi} + 1] .
\]
Next by Freedman's inequality (picking $a = \beta_{n,h}/2$ and $b=\sqrt{\beta_{n,h}}$), we have,
\begin{align*}
\PP\bigg[\sum_{(n', h')<(n,h)} G_{n', h'} \ge  \beta_{n,h}/2 \bigg]
 &= \PP\bigg[\sum_{(n', h')<(k,h)} G_{n', h'} \ge  \beta_{n,h}/2 \text{ and }V_{n,h}\le  v_{n,h}\bigg]\\
 &\le \exp\Bigg(\frac{-\beta_{n,h}^2/4}{2v_{n,h} + \beta_{n,h}^{3/2}/3}\Bigg)\\
 &\le \exp\Bigg[\max\Bigg(\frac{-\beta_{n,h}^2}{16v_{n,h}},-\frac{3}{8}\sqrt{\beta_{n,h}}\bigg)\Bigg]\\
 &
 \le \frac{\delta}{(nH)^2\pi^2}
\end{align*} 
provided
\begin{align}
\beta_{n,h} &=  
c\cdot {C'_{\psi}}^2 \ln\det(A_{n,h})\ln(nH/\delta)
\end{align}
for some sufficiently large constant $c>0$ such that
\[
\beta_{n,h}^2\ge 32v_{n,h}\cdot\ln(\pi nH/\delta) + 8^4\ln^4(\pi nH/\delta)
\]
Therefore, by a union bound,
\begin{align*}
\PP\bigg[\sum_{(n', h')<(n,h)} G_{n', h'} \ge  \beta_{n,h}/2 \text{ for some } n,h\bigg]
&\le \sum_{(n,h)< (\infty, H)} \PP\bigg[\sum_{(n', h')<(n,h)} G_{n', h'} \ge  \beta_{n,h}/2 \bigg]\\
&\le \sum_{n=1}^{\infty}\frac{H\delta}{n^2H^2\pi^2}\\
&\le \delta \cdot \frac{\pi^2}{6\pi^2}\cdot \frac{1}{H}\\
&\le \delta.
\end{align*}
This completes the proof.
\end{proof}
Lastly, we show that $E_n=1$ holds with high probability through out.
\begin{lemma}[Confidence Ball]
	\label{lemma:confidence}
Let $\delta > 0$.
Then 
\[
\PP\big[\forall n,h\le (N,H), \quad E_{n,h}=1\big] \ge 1-\delta. 
\]
\end{lemma}
\begin{proof}
We will show that $E_{n,h} = 1$ for all $(n,h)\le (N,H)$ given  $\sum_{(n', h')\le(n,h)} G_{n', h'} \le   \beta_{n,h}/2$ for all $(n,h)\le (N,H)$, as this proves the lemma.
We show this by induction on $n,h$.
For the base case $(n,h)= (1,1)$, we have $Z_{1,1} = \|M^*\|_{F}^2 \le \beta_{1,1}$.
By inductive hypothesis, $E_{n', h'} = 1$ for all $(n', h') < (n,h)$. 
By Lemma~\ref{lemma:zkh}, we have
\begin{align*}
Z_{n,h} &\le 
\|M^*\|_F^2 + 2 \sum_{(n', h')<(n,h)}\frac{\phi_{n', h'}^\top(M_{n,h}-M^*)\eta_{n', h'}}{1+\tw_{n', h'}^2}
+ \sum_{(n', h')<(n,h)}
\|\eta_{n', h'}\|_2^2 \frac{\tw_{n', h'}^2}{1+\tw_{n', h'}^2}\\
&\le \|M^*\|_F^2 + \beta_{n,h}/2 + \sum_{(n', h')<(n,h)}
\|\eta_{n', h'}\|_2^2 \frac{\tw_{n', h'}^2}{1+\tw_{n', h'}^2}\\
&\le \|M^*\|_F^2 + \beta_{n,h}/2 + \sum_{(n', h')<(n,h)} {C_{\psi}'}^2 \min(\tw_{n', h'}^2, 1)\\
&\le \|M^*\|_F^2 + \beta_{n,h}/2 +  2 {C_{\psi}'}^2 \ln\det(A_{n,h}) 
\\
&\le \beta_{n,h},
\end{align*}
provided
\[
\beta_{n,h}/2 \ge \|M^*\|_F^2 
+  2{C_{\psi}'}^2 \ln\det(A_{n,h})
\]
or 
\[
\beta_{n,h}/2 \ge \|M^*\|_F^2 
+  2d{C'_{\psi}}^2\ln[(n-1)HC_{\phi} + (h-1)C_{\phi} + 1]
= c[C_{M}+{C'_{\psi}}^2\cdot \ln(nHC_{\phi})]\cdot d
\]
for some sufficiently large constant $c>0$.
\end{proof}
\subsection{Proof of Theorem~\ref{thm:regret-bound}}
\label{sec:full-proof-regret}
Before we prove Theorem~\ref{thm:regret-bound}, we first prove the Lemma~\ref{lemma:sum-confidence}
\begin{proof}[Proof of Lemma~\ref{lemma:sum-confidence}]
	Note that 
	\[
	\sum_{(n', h')\le (N,H)}\min(1,w_{n', h'}^2)\le 2\sum_{(n', h')\le (N,H)} \ln(1+w_{n', h'}^2).
	\]
	We will bound the right hand side by establishing an inequality with $\ln\det(A_{N+1})$. 
	Recall that 
	\[
	A_{n+1} = A_n + \sum_{h=1}^{H} \phi_{n,h}\phi_{n,h}^\top.
	\]
	We have
	\begin{align*}
	\det(A_{n+1}) &= \det(A_n)\cdot\det\Big(I + \sqrt{A_{n}^{-1}}\sum_{h=1}^{H} \phi_{n,h}\phi_{n,h}^\top\sqrt{A_{n}^{-1}}\Big).
	\end{align*}
	Notice that each eigenvalue $\lambda_i$ of 
	$I + \sqrt{A_{n}^{-1}}\sum_{h=1}^{H} \phi_{n,h}\phi_{n,h}^\top\sqrt{A_{n}^{-1}}$ is at least $1$ and 
	\[
	\sum_{i=1}^d(\lambda_i-1) = \tr\Big(I + \sqrt{A_{n}^{-1}}\sum_{h=1}^{H} \phi_{n,h}\phi_{n,h}^\top\sqrt{A_{n}^{-1}}\Big) -d
	 = \sum_{h}w_{n,h}^2.
	\]
	Therefore,
	\begin{align*}
	\det\Big(I + \sqrt{A_{n}^{-1}}\sum_{h=1}^{H} \phi_{n,h}\phi_{n,h}^\top\sqrt{A_{n}^{-1}}\Big) =
	\prod_{i=1}^d \lambda_i = \prod_{i=1}^d [1+(\lambda_i-1)] \ge 1 + \sum_{i=1}^{d}(\lambda_i-1) = 1+\sum_{h}w_{n,h}^2.
	\end{align*}
	Moreover, 
	\begin{align*}
	 1+\sum_{h}w_{n,h}^2 = \frac{\sum_{h=1}^{H}(1+Hw_{n,h}^2)}{H}
	 \ge \prod_{h=1}^{H}(1+Hw_{n,h}^2)^{1/H}\ge  \prod_{h=1}^{H}(1+w_{n,h}^2)^{1/H}.
	\end{align*}
	Thus we have
	\begin{align*}
	\det(A_{n+1})&\ge \det(A_n)\prod_{h=1}^{H}(1+w_{n,h}^2)^{1/H}\ge \det(A_{n-1})\cdot \ldots \\
	&\ge \det(A_1) \prod_{n'=1}^{n}\prod_{h=1}^{H}(1+w_{n',h}^2)^{1/H}.
	\end{align*}
	This gives 
	\[
	\sum_{(n', h')\le (N,H)} \ln(1+w_{n', h}^2) 
	\le H\ln\det(A_{N+1}) \le Hd\ln(NHC_{\phi} + 1)
	\]
	completing the proof.
\end{proof}

We are now ready to prove Theorem~\ref{thm:regret-bound}.
\begin{proof}[Proof of Theorem~\ref{thm:regret-bound}]
We let $\delta \le 1/(NH)$.
By Lemma~\ref{lemma:contrencetion lemma} and \ref{lemma:confidence}, we pick
\[
\beta_{n}\ge 
c[C_{M}+{C'_{\psi}}^2\cdot \ln(nHC_{\phi})]\cdot d
+ c\cdot {C_{\psi}'}^2 \cdot d\cdot \ln(nHC_{\phi}/\delta)
 = \Theta(C_{M}+ {C_{\psi}'}^2)\cdot\ln(nHC_{\phi})\cdot d.
\]
Then the following is guaranteed,
\[
\Pr[\forall n\le N: \quad E_n= 1] \ge 1-\delta.
\]
Then by Lemma~\ref{lemma:conitioning regret}, we have,
\begin{align*}
\reg(T)&\le 
2C_{\psi}H\sqrt{d\beta_{N}} \cdot \sqrt{HN\cdot 2Hd\cdot \ln[NHC_{\phi} + 1]}
+ O(1)\\
&= O\Big[C_{\psi}\sqrt{C_M+ {C_{\psi}'}^2}\cdot \ln(NHC_{\phi})\Big]\cdot \sqrt{d^3H^3 T}.
\end{align*}
\end{proof}

\subsection{Proof of Theorem~\ref{thm:regret-bound-1}}
\begin{proof}[Proof of Theorem~\ref{thm:regret-bound-1}]
The proof of Theorem~\ref{thm:regret-bound-1} is nearly identical with that of Theorem~\ref{thm:regret-bound}. We will   modify Lemma~\ref{lemma:ball upper bound} and Lemma~\ref{lemma:induction} to counter for the change of the confidence ball.

For a modification of Lemma~\ref{lemma:ball upper bound}, we show that for any $M\in B^{(2)}_n$,
\[
\|\phi_{s,a}^\top(M-M_n)\|_2
\le \|\phi_{s,a}^\top(A_{n})^{-1/2}\|_2\|(A_{n})^{1/2}(M-M_{n})\|_F \le \sqrt{\beta_{n} \cdot \phi(s,a)^\top (A_{n})^{-1} \phi(s,a)}.
\]

For a modification of Lemma~\ref{lemma:induction}, we have
\begin{align*}
Q_{n,h}(s_{n,h}, a_{n,h}) &- 
\Big[r(s_{n,h}, a_{n,h}) + P(\cdot|s_{n,h}, a_{n,h})^\top V_{n,h+1}\Big]\\
&=  \phi_{n,h}^\top (\wt{M} - M^*) \Psi^\top 
V_{n,h+1}\\
&\le \| \phi_{n,h}^\top (\wt{M} - M^*)\|_2\cdot 
\|\Psi^\top 
V_{n,h+1}\|_{2}\\
&\le  C_{\psi} H\| \phi_{n,h}^\top (\wt{M} - M^*)\|_2\\
&\le C_{\psi}H ( \| \phi_{n,h}^\top (\wt{M} - M_{n})\|_2 +  \| \phi_{n,h}^\top ( M_{n} - M^*)\|_2)\\
&\le 2C_{\psi}H\sqrt{\beta_{n}\cdot \phi_{n,h}^\top (A_{n})^{-1}\phi_{n,h}}.
\end{align*}

The rest of the proof follows analogously from that of Theorem~\ref{thm:regret-bound}. 
\end{proof}

\section{Derivation of Kernelization}
\label{sec:derive-kernel}
\paragraph{Notations}
For the analysis, let us presumably have the features $\phi, \psi$. 
In the actual algorithm we will avoid using features directly.
At time $t=(n-1)H + h$, we denote
\begin{align}
\Phi_{n,h} = [\phi_{1,1},\phi_{1,2}, \ldots, \phi_{n,h}]^\top\quad&\text{and}\quad 
\Psi_{n,h} = [\psi_{1,1},\psi_{1,2}, \ldots, \psi_{n,h}]^\top\nonumber\\
\Psi = [\psi(s_1),\psi(s_2), \ldots, \psi(s_{|\cS|})]^\top
\quad&\text{and}\quad K_{\psi} = \Psi^\top \Psi.
\end{align}
Note that $\Phi_{n,h}$ and $\Psi_{n,h}$ are the feature vectors of the encountered state-action pairs; $\Psi$ are the features of all states in $\cS$.
We also overload the notation by denoting $\Phi_{n} = \Phi_{n, H}$ and $\Psi_{n} = \Psi_{n, H}$.
We then observe
\begin{align*}
\bK_{\phi_{n,h}} &= \Phi_{n,h}^\top \Phi_{n,h}\in \RR^{t\times t},\quad 
\bK_{\psi_{n,h}} = \Psi_{n,h}^\top \Psi_{n,h}\in \RR^{t\times t},\quad \text{and}\quad 
\wb{\bK}_{\psi_{n,h}} = \Psi_{n,h}^\top\Psi\in \RR^{t\times |\cS|}\\
\end{align*}
Note that 
\[
\bK_{\phi_{n,h}}[(n_1, h_1), (n_2, h_2)] = \phi_{n_1, h_1}^\top \phi_{n_2, h_2} = k_{\phi}[(s_{n_1, h_1}, a_{n_1, h_1}), (s_{n_2, h_2}, a_{n_2, h_2})],
\]
can be represented without knowing the features. Similarly, we do not need features to compute $\bK_{\psi,n,h}$ and $\wb{\bK}_{\psi,n,h}$.

\paragraph{Kernelized Value Estimation}
For a given vector $V\in \RR^{\cS}$, let us use kernel to represent the prediction
\[
\phi(s,a)^\top M_{n}\Psi^\top V.
\] 
We introduce a dual matrix $\Theta_{n}$, and let $M_{n} = \Phi_{n-1}^\top \Theta_{n} \Psi_{n-1}$.
Then we have
\[
\phi(s,a)^\top M_{n}\Psi^\top V 
= \phi(s,a)^\top \Phi_{n-1}^\top \Theta_{n} \Psi_{kn-1}\Psi^\top V = \bk_{\Phi_{n-1}, s,a}^\top \Theta_{n} 
\wb{\bK}_{\Psi_{n}} V,
\]
where 
\[
\bk_{\Phi_{n,h}, s,a} = \Phi_{n,h}\phi(s,a)
= [k_{\phi}\big((s_{1,1} a_{1,1}), (s,a)), k_{\phi}((s_{1,2}, a_{1,2}), (s,a)), \ldots  k_{\phi}((s_{n,h}, a_{n,h}), (s,a))]^\top\in  \RR^{t}.
\]
Therefore, it remains to represent $\Theta_{n+1}$ by the kernel matrices.

Recall that 
\[
M_{k+1} = (I+\Phi_{n}^\top\Phi_{n})^{-1} \sum_{(n', h')\le (n,h)} \phi_{(n',h')} \psi_{n',h'}^\top K_{\psi}^{-1}.
\]
We can write it as
\[
(I+\Phi_{n}^\top\Phi_n) \Phi_n^\top \Theta_{n+1} \Psi_{n} \Psi^\top \Psi = \Phi_{n}^\top\Psi_{n}.
\]
Rearranging the terms, we obtain
\begin{align*}
\Theta_{n+1} = (I+\Phi_{n}\Phi_{n}^\top)^{-1}
\Psi_{n}\Psi_{n}^\top [(\Psi_{n}\Psi^\top) (\Psi \Psi_{n}^\top)]^{-1} = 
(I+\bK_{\Phi_{n}})^{-1} \bK_{\Psi_{n}} (\wb{\bK}_{\Psi_{n}}\wb{\bK}_{\Psi_{n}}^\top)^{-1}.
\end{align*}
This completes the kernelization of the prediction.

\paragraph{Kernelized Confidence Bound}
Next we write the confidence bound in the kernelized way as well. We represent $w_{n,h}$ the same way as in \cite{Valko}.
Note that
\[
(I+\Phi_{n-1}^\top \Phi_{n-1})\phi(s,a)
= \phi(s,a) + \Phi_{n-1}^\top \bk_{\Phi_{n-1},s,a}.
\] 
We can write $\phi(s,a)$ as
\begin{align*}
\phi(s,a) &= (I+\Phi_{n-1}^\top \Phi_{n-1})^{-1}\phi(s,a)
+(I+\Phi_{n-1}^\top \Phi_{n-1})^{-1}\Phi_{n-1}^\top \bk_{\Phi_{n-1},s,a}\\
&=  (I+\Phi_{n-1}^\top \Phi_{n-1})^{-1}\phi(s,a)
+\Phi_{n-1}^\top(I+\Phi_{n-1} \Phi_{n-1}^\top)^{-1} \bk_{\Phi_{n-1},s,a}.
\end{align*}
Therefore
\begin{align*}
\phi(s,a)^\top\phi(s,a)
&=\phi(s,a)^\top(I+\Phi_{n-1}^\top \Phi_{n-1})^{-1}\phi(s,a)
+\phi(s,a)^\top\Phi_{n-1}^\top(I+\Phi_{n-1} \Phi_{n-1}^\top)^{-1} \bk_{\Phi_{n-1},s,a}\\
&= \phi(s,a)^\top(I+\Phi_{n-1}^\top \Phi_{n-1})^{-1}\phi(s,a) + \bk_{\Phi_{n-1},s,a}^\top (I+\bK_{\Phi_{n-1}})^{-1}\bk_{\Phi_{n-1},s,a},
\end{align*}
from which we solve $w_{n}(s,a)^2:=\phi(s,a)^\top(I+\Phi_{n-1}^\top \Phi_{kn-1})^{-1}\phi(s,a)$:
\[
w_{n}(s,a)^2 = k[(s,a),(s,a)]
- \bk_{\Phi_{n-1},s,a}^\top (I+\bK_{\Phi_{n-1}})^{-1}\bk_{\Phi_{n-1},s,a}.
\]

\paragraph{Kernelized Algorithm}
Denote 
\begin{align*}
x_{n}(s,a)^\top
&= \phi(s,a)^\top \Phi_{n-1}^\top (I+\bK_{\Phi_{n-1}})^{-1} \bK_{\Psi_{n-1}} (\wb{\bK}_{\Psi_{n-1}}\wb{\bK}_{\Psi_{n-1}}^\top)^{-1} \wb{\bK}_{\Psi_{n}}\\
&= \bk_{\Phi_{n-1},s,a}^\top (I+\bK_{\Phi_{n-1}})^{-1} \bK_{\Psi_{n-1}} (\wb{\bK}_{\Psi_{n-1}}\wb{\bK}_{\Psi_{n-1}}^\top)^{-1} \wb{\bK}_{\Psi_{n}}
\end{align*}
We are now ready to write our $Q$-function estimator:
\begin{align}
\forall (s,a) \in \cS\times \cA&:\quad 
Q_{n,H+1}(s,a) = 0 \quad\text{and}\quad \nonumber\\
\forall h\in [H]&: \quad
Q_{n,h}(s,a) = r(s,a) + x_{n}(s,a)^\top V_{n,h+1}
+ \eta_{n} w_{n}(s,a),
\end{align}
where $\eta_n$ is a parameter to be determined.
With \eqref{eqn:compute q} replaced by \eqref{eqn:compute q-kernel} in Algorithm~\ref{alg:core-rl}, we obtain our \emph{Kernelized MatrixRL}, algorithm Algorithm~\ref{alg:core-rl-kernel}.

\subsection{Proof of Theorem~\ref{thm:regret-bound-kernel}}
\begin{proof}[Proof of Theorem~\ref{thm:regret-bound-kernel}]
	To prove the theorem, let us presumably have access to some features $\phi$ and $\psi$, which are of dimension $d$ and $d'$, respective.
	For the infinite case we can take $d,d'\rightarrow\infty$ and since the complexity does not depending on $d,d'$ our proof still follows. 
	
	For simplicity, let us define $\cX=\cS\times \cA$.
	Firstly, we notice that the quantity $\|P\|_{\cH_{\phi}\times \cH_{\psi}} $ is equivalent to $\|M\|_F$ in the finite dimensional setting. 
	Let us define $ C_{M} := \|P\|_{\cH_{\phi}\times \cH_{\psi}}$.
	Since
	\[
	\sum_{s\in \cS} \psi(s)\psi(s)^\top  = I \in \RR^{d'\times d'}
	\]
	Now, without changing our algorithm, we have that
	\[
	\|\Psi K_{\psi}^{-1}\|_{2,\infty} \le 1.
	\]
	Furthermore, 
	\[
	\forall v\in \cH_{\psi}:\quad
	\|\Psi^\top v\|_2= \|v\|_H \le C_{\psi} \|v\|_{\infty}
	\] 
	Thus all the conditions except $C_{\phi}$ in \ref{assump:linrl-regularity-1} are satisfied ($C_{\psi}'=1$). 
	It will become clear that $C_{\phi}$ is absorbed in the definition of $\wt{d}$.
%
%

	Since Kernelized MatrixRL is equivalent to MatrixRL when the feature space is finite dimensional,
	the proof of Theorem~\ref{thm:regret-bound-kernel} is nearly identical with that of Theorem~\ref{thm:regret-bound-1}.	
	To introduce the dependence of the kernel complexity, we keep $\beta_n$ as the following form:
	\[
	\beta_{n} = c\cdot(\ln \det(A_{n}) \cdot \ln(nH) + C_M^2).
	\]
	By following from the steps of the proof of  Theorem~\ref{thm:regret-bound-1}, 
	we obtain that
	\[
	\reg(T)\le 
	2C_{\psi}H\sqrt{\beta_{N}} \cdot \sqrt{HN\cdot  2H\ln\det(A_{N+1})}.
	\]
	Note that
	\[
	\det(A_{N+1}) = \det(I + \Phi_{n}^\top\Phi_{n}).
	\]
	Since $ \Phi_{n}^\top\Phi_{n}$ has the same non-zero eigenvalues with  that of  $\Phi_{n}\Phi_{n}^\top = \bK_{\Phi_{n}}$, we have
	\[
	\det(A_{N+1}) = \det(I + \Phi_{n}^\top\Phi_{n})\le \wt{d}\log(NH),
	\]
	Thus we have\[
	\reg(T)\le O\Big(C_{\psi}C_M\cdot \log(T)\cdot \wt{d}(k_{\phi})\cdot H^2\cdot\sqrt{T}\Big)
	\]
	as desired. 
\end{proof}
\end{document}